\documentclass{article}

\usepackage{arxiv}

\usepackage[utf8]{inputenc} % allow utf-8 input
\usepackage[T1]{fontenc}    % use 8-bit T1 fonts
\usepackage{hyperref}       % hyperlinks
\usepackage{url}            % simple URL typesetting
\usepackage{booktabs}       % professional-quality tables
\usepackage{amsfonts}       % blackboard math symbols
\usepackage{nicefrac}       % compact symbols for 1/2, etc.
\usepackage{microtype}      % microtypography
\usepackage{graphicx}
\usepackage{amsmath}
\usepackage{amssymb}
\usepackage{tcolorbox}
\usepackage{soul}
\usepackage{amsthm}
\usepackage{multirow}
\newtheorem{theorem*}{Theorem}
\newtheorem{corollary*}{Corollary}[theorem*]
\newtheorem{lemma*}[theorem*]{Lemma}

\hypersetup{
    colorlinks=true,
    linkcolor=blue,
    filecolor=blue,      
    urlcolor=blue,
}

\title{Topological Gradient-based Competitive Learning}
\author{
  Pietro~Barbiero\hspace{1mm}\href{https://orcid.org/0000-0003-3155-2564}{\includegraphics[scale=0.06]{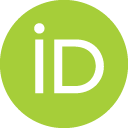}}\\
  Cambridge University\\
  United Kingdom\\
  \texttt{barbiero@tutanota.com} \\
   \And
 Gabriele Ciravegna\hspace{1mm}\href{https://orcid.org/0000-0002-6799-1043}{\includegraphics[scale=0.06]{orcid.png}}\\
  University of Florence\\
  Italy\\
   \And
 Vincenzo Randazzo\hspace{1mm}\href{https://orcid.org/0000-0003-3640-8561}{\includegraphics[scale=0.06]{orcid.png}}\\
  Politecnico di Torino\\
  Italy\\
   \And
 Giansalvo Cirrincione\hspace{1mm}\href{https://orcid.org/0000-0002-2894-4164}{\includegraphics[scale=0.06]{orcid.png}}\\
  University of Picardie Jules Verne\\
  France\\
  University of South Pacific\\
  Fiji
}

\begin{document}
\maketitle

{\let\thefootnote\relax\footnote{P. Barbiero conceived the main idea and the dual method. G. Cirrincione developed the theory and validated theorems, assumptions, and proofs with P. Barbiero. G. Ciravegna and P. Barbiero conceived, planned, and carried out the experiments. All authors discussed the results and contributed to the final manuscript.}}

\begin{abstract}
\textit{Topological learning} is a wide research area aiming at uncovering the mutual spatial relationships between the elements of a set. Some of the most common and oldest approaches involve the use of unsupervised competitive neural networks. However, these methods are not based on gradient optimization which has been proven to provide striking results in feature extraction also in unsupervised learning. Unfortunately, by focusing mostly on algorithmic efficiency and accuracy, deep clustering techniques are composed of overly complex feature extractors, while using trivial algorithms in their top layer. The aim of this work is to present a novel comprehensive theory aspiring at bridging competitive learning with gradient-based learning, thus allowing the use of extremely powerful deep neural networks for feature extraction and projection combined with the remarkable flexibility and expressiveness of competitive learning. 
% The theory is based on the intuition that neural networks are able to learn topological structures by working directly on the transpose of the input matrix, 
% Such a theory has a great potential as it shows how neural networks are able to learn topological structures by working on the transpose of the input matrix i.e. $X^T$, leading to better results on high-dimensional scenarios. 
In this paper we fully demonstrate the theoretical equivalence of two novel gradient-based competitive layers. Preliminary experiments show how the dual approach, trained on the transpose of the input matrix i.e. $X^T$, lead to faster convergence rate and higher training accuracy both in low and high-dimensional scenarios.

\end{abstract}

% keywords can be removed

\section{Introduction}
From the dawn of Artificial Intelligence (AI), data clustering has always been a field of great interest in the scientific community. First approaches were mainly based on similarity measures among data. 
Prominent methods like k-Means \cite{macqueen1967some}, Gaussian Mixture Models (GMM) \cite{mclachlan1988mixture} and more recently Density Based Spatial Clustering (DBSCAN) \cite{ester1996density} have been extensively used to uncover unknown relations in unsupervised problems. 
These types of approaches are capable of finding groups of samples that are similar, but they cannot detect the underlying topology. Hierarchical clustering partially solved this issue by creating a hierarchy of clusters either with an agglomerative \cite{sibson1973slink} \cite{defays1977efficient} or with a divisive strategy \cite{kaufman2009finding} \cite{cirrincione2020gh}. Other approaches, instead, try to solve this problem by introducing a topological structure among cluster nodes. The first algorithm exploiting this concept is the Self-Organizing-Map (SOM) by Kohonen \cite{kohonen1982self}, where a neural network is trained to represent the input space using a grid, whose number of units and their connections, i.e. the topology, is defined in advance. 
Alternatively, techniques such as Neural Gas (NG) \cite{martinetz1991neural}, Growing Neural Gas (GNG) \cite{fritzke1995growing} and their variants \cite{fritzke1997self,ghng,GHBNG} apply the the Competitive Hebbian Learning (CHL) \cite{hebb2005organization,martinetz1993competitive,chl} for defining local topology; indeed, given an input sample, the two closest neurons, called first and second winners, are linked by an edge \cite{martinetz1994topology,fritzke1997some}. All the previous cited methods belong to the \textit{competitive learning} field; here, units compete to represent the input sample, i.e. they move towards it depending on their distances and the network current topology (neighbourhoods). % Such an approach is considered more biological than non-competitive learning, e.g. the k-Means algorithm, which depends only on the computation of various partial derivatives \cite{somqiang2010survey}. An overview can be found in \cite{somqiang2010survey}.
Another issue concerning the above-cited methods is the well-known curse of dimensionality \cite{barbiero2020modeling,altman2018curse}. Euclidean measures are no more effective when dealing with high-dimensional data such as images. To this aim, many works proposed dimensionality reduction and feature extraction methods as pre-processing before the clustering step like Principal Component Analysis (PCA) \cite{pearson1901liii} and kernel functions. These methods are indeed capable of mapping row data into a feature space with a much lower dimensionality. However, the effectiveness of such techniques is limited when dealing with complex latent structures.    
Recently, however, Deep Neural Networks (DNN), and more specifically  Convolutional Neural Network (CNN) \cite{lecun1989backpropagation}, have incredibly improved processing performances when dealing with highly-dimensional data in supervised learning. As a consequence, many approaches tried to apply these methods also to the unsupervised learning field. Deep neural networks are capable to transform high-dimensional data into clustering-friendly representations. By employing DNN, clustering and feature transformation are now treated as a single task.
DNN architecture may be directly trained through the optimization of a clustering loss. The choice of the learning function is particularly important when dealing with this type of architecture. As a matter of fact, straightforward employment of DNN may lead to corrupted feature transformation, where data are mapped to compact clusters that do not reflect the real data topology. In order to overcome this issue, some works proposed to exploit both unsupervised and supervised network pre-training, weight regularization, and data augmentation techniques. 
For what concerns unsupervised network pre-training, common strategies consider training Restricted Boltzmann Machines (RBM) \cite{smolensky1986information} or AutoEncoders (AE) \cite{hinton1994autoencoders} and later fine-tune the networks (only the encoder for AE) through a clustering learning function only \cite{xie2016unsupervised} \cite{chen2015deep}.
Supervised pre-training techniques are instead commonly employed when dealing with image data. Indeed, classical clustering algorithms perform well when using the feature extracted from the last layer of a CNN, pre-trained on big image dataset as ImageNet \cite{hsu2017cnn}.
Direct approaches that do not consider any network pre-training, have been recently proposed in \cite{hu2017learning}, \cite{yang2016joint}, \cite{chang2017deep}.
% \textcolor{red}{Add a few words for each method...}
% \paragraph{Information Maximizing Self-Augmented Training (IMSAT) }
% \paragraph{Joint Unsupervised LEarning (JULE)}
% \paragraph{Deep Adaptive Image Clustering (DAC)}
Otherwise, clustering learning procedures may be integrated with a network learning process. This allows the employment of more complex architectures like Autoencoders (AE), Variational-Autoencoders (VAE) or Generative Adversarial Networks (GAN). Such techniques commonly consider a double stage learning in which they first learn a good representation of the input space through a network loss function and later fine tune the network by also optimizing a clustering-specific loss.
% Such techniques commonly consider a double stage learning in which they first learn a good representation of the input space through the optimization of a network loss function. Afterwards, a fine-tuning step is performed where the optimization of the previous loss function is integrated with the optimization of a clustering-specific loss. 
% \textcolor{red}{Add a few examples also here}.

To the best of our knowledge, however, no previous work proposed to join the strength of DNN with the higher representation capabilities of competitive learning approaches. In this work, we study two possible variants of a neural network architecture in which competitive learning is taken into consideration by the loss function. The proposed architectures can either be employed by themselves or they can be placed on top of more complex neural architectures such as AE, CNN, VAE or GAN. 

% \textcolor{red}{Say something about working with the transposed matrix too...}

% \begin{tcolorbox}[colback=blue!5!white,colframe=blue!75!black,title=Main contributions]
%   \begin{itemize}
%     \item The \textit{gradient-based competitive layer} and the \textit{dual GBC layer}, two novel approaches which use backpropagation for clustering.
%     \item A novel theory describing the equivalence between the GBC layer and its dual.
%     \item Description of how CHL can be integrated in the gradient-based learning framework by using both the GBC layer and its dual.
%     \item Experiments showing: (i) the benefits of the proposed approaches w.r.t. non-topological clustering, and (ii) the superior stability and rate of convergence of the dual network w.r.t. the GBC layer.
%     % \item Summary of the main potentialities of the approach and description of how the theory can be further extended to more complex topological learning tasks.
% \end{itemize}
% \end{tcolorbox}

This work is organized in three main sections. The first one %Section \ref{sec:methods} 
describes two novel methods that can be used to join competitive and gradient-based learning, namely the \textit{gradient-based competitive layer} (GBC layer) and the \textit{dual GBC layer} (DGBC layer). The following section %Section \ref{sec:experiments} 
presents preliminary experiments showing the benefits and the differences of the two approaches. Finally, the last section %Section \ref{sec:extensions} 
describes how the methods presented in this work can be further developed and extended.

\section{Gradient-based competitive layers} \label{sec:methods}

This section describes two different approaches that join competitive and gradient-based learning.
In a standard competitive layer \cite{rumelhart1985feature,barlow1989unsupervised,haykin2007neural}, every competing neuron is described by a vector of weights $w_i$, representing the position of the neuron (a.k.a. \textit{prototype}) in the input space. The inverse of the Euclidean distance between the input data $x_k$ and the weight vector $w_i$ represents the similarity between the input and the prototype. For every input vector $x_k$, the prototypes \textit{compete} with each other to see which one of them is the most similar to that particular input vector. By following the Competitive Hebbian Learning (CHL) rule \cite{hebb2005organization,martinetz1993competitive}, the two closest prototypes to $x_k$ are connected using an edge, representing their mutual activation.

In general, competitive learning is based on more or less heuristic rules. Instead, the family of k-Means algorithms is justified by the minimization of a loss function, representing the quantization error.
The first proposed approach (GBC layer), instead, is based directly on the minimization of this loss. This is performed by using a first-order gradient technique, for straightly estimating the prototypes.
Adding this layer to the top of a deep neural network, a deep clustering can be performed by backpropagating the gradient information from the clustering to the previous layers.
As a consequence, the benefits of using a powerful feature extractor and a sophisticated topological learning algorithm can be both exploited. 
However, the loss used by this approach is only function of the training set and the weights of the layer (the output neuron weights are the prototypes). This means the outputs of the layer are not taken into account. In this sense, the first approach is better interpreted as a straight competitive learning on the input set than a true layer to be added.
The second approach (dual GBC layer) is more \textit{neural}, because it represents a true transformation of the inputs.
It is an alternative approach for the implementation of a competitive layer which is trained using the transpose of the input matrix, i.e. $X^T$.

%\textcolor{red}{Section \ref{sec:transpose} describes the intuitions behind the dual approach. }
%The general relationship between a base layer (as GBC) and its dual layer (as DGBC) is studied in Section \ref{sec:duality}, where the theoretical equivalence of the two architectures is investigated.
%Section \ref{sec:analysis} analyzes the training of the GBC and the DGBC layers.

\subsection{A new insight in neural theory} \label{sec:transpose}

A deep neural network can be interpreted as a nonlinear function $f$ mapping input data $x \in \mathbb{R}^d$ into a different representation $y \in \mathbb{R}^p$ which is optimized according to an error function $\mathcal{L}$. Hence, a concise representation of a neural network is a pair $(f, \mathcal{L})$ such that:
% \begin{eqnarray}
%     &f&: \mathbb{R}^d \rightarrow \mathbb{R}^p \nonumber \\
%     &\mathcal{L}&: \cdot \rightarrow \mathbb{R}
% \end{eqnarray}
\begin{equation}
    f: \mathbb{R}^d \rightarrow \mathbb{R}^p, \quad \mathcal{L}: \cdot \rightarrow \mathbb{R}
\end{equation}
The relationship between $d$ and $p$ and the kind of loss function used to train the model determine the kind of learning task. 
In most settings, deep neural networks are used to contract the input space into an interpretable codomain where the performance of the network can be easily assessed. For instance, if $p=1$ and the loss function is the mean squared error between the output of the network ($y \in \mathbb{R}$) and a target variable ($t \in \mathbb{R}$), the learning task is called regression:
% \begin{eqnarray}
%     &f&: \mathbb{R}^d \rightarrow \mathbb{R} \nonumber \\
%     &\mathcal{L}& = mse(y,t)
% \end{eqnarray}
\begin{equation}
    f: \mathbb{R}^d \rightarrow \mathbb{R}, \quad \mathcal{L}= mse(y,t) 
\end{equation}
Another common learning task is classification which can be obtained by setting $p=c$, where $c$ corresponds to the number of classes, and by using a cross-entropy error function $H$:
% \begin{eqnarray}
%     &f&: \mathbb{R}^d \rightarrow \mathbb{R}^c \nonumber \\
%     &\mathcal{L}& = H(y,t)
% \end{eqnarray}
\begin{equation}
    f: \mathbb{R}^d \rightarrow \mathbb{R}^c, \quad
    \mathcal{L}= H(y,t)
\end{equation}
In practice, the training process is performed using a dataset composed of objects $X \in \mathbb{R}^{n,d}$ and (for supervised tasks) targets $T \in \mathbb{R}^{n,p}$, where $n$ is the number of samples, $d$ the number of input features, and $p$ the number of output features. Hence, the result of the training process can be seen as a projection of the input matrix $X$ into a (usually) lower dimensional representation $Y$:
\begin{eqnarray}
    &f&: \mathbb{R}^{n,d} \rightarrow \mathbb{R}^{n,p} \qquad [\text{usually } p \ll d] \nonumber \\
    &Y& = f(X)
\end{eqnarray}
Therefore, a deep neural network can be summarized as a nonlinear map \textit{reducing} the number of columns of a matrix, while keeping the original number of rows.
If the rows of the input matrix represent a set of samples and the columns a set of features (as it usually is), then the neural network is actually shrinking the number of features.

However, if we consider the transpose problem, where $X^T \in \mathbb{R}^{d,n}$, the neural network can still be used to transform the input matrix into useful representations. Also in this case, the output must keep the same number of rows of the input by construction, i.e. $Y \in \mathbb{R}^{d,k}$. If $k<n$ the neural network is \textit{contracting} the input, while if $k>n$ the neural network is \textit{augmenting} the input.
While normally neural networks are used to generate an abstract representation of the input features useful for supervised tasks like classification or regression, the transpose problem can be used to generate an abstract representation of the input samples useful for learning the topology of the input manifold. In fact, by choosing an appropriate clustering error function $\mathcal{C}$ we can define a deep neural network learning the positions of cluster centroids (a.k.a. \textit{prototypes}) as:
% \begin{eqnarray}
%     &f&: \mathbb{R}^{d,n} \rightarrow \mathbb{R}^{d,k} \nonumber \\
%     &\mathcal{L}& = \mathcal{C}
% \end{eqnarray}
\begin{equation}
    f: \mathbb{R}^{d,n} \rightarrow \mathbb{R}^{d,k}, \quad \mathcal{L} = \mathcal{C}
\end{equation}
where $k$ corresponds to the number of output units of the network. Each of the $k$ output unit returns as output a $q^T \in \mathbb{R}^d$ vector representing a position in the feature space $\mathbb{R}^d$. Hence, by optimizing such positions according to a clustering error function, the neural network can learn prototypes corresponding to cluster centroids.

\subsection{Duality theory} \label{sec:duality}

The intuitions outlined in the previous section can be formalized in a general theory which considers the duality properties between a linear single-layer neural network and its dual, defined as a network which learns on the transpose of the input matrix and has the same number of neurons.

\begin{figure}[th]
    \centering
    \includegraphics[width=0.95\columnwidth, trim={100 60 100 40}, clip]{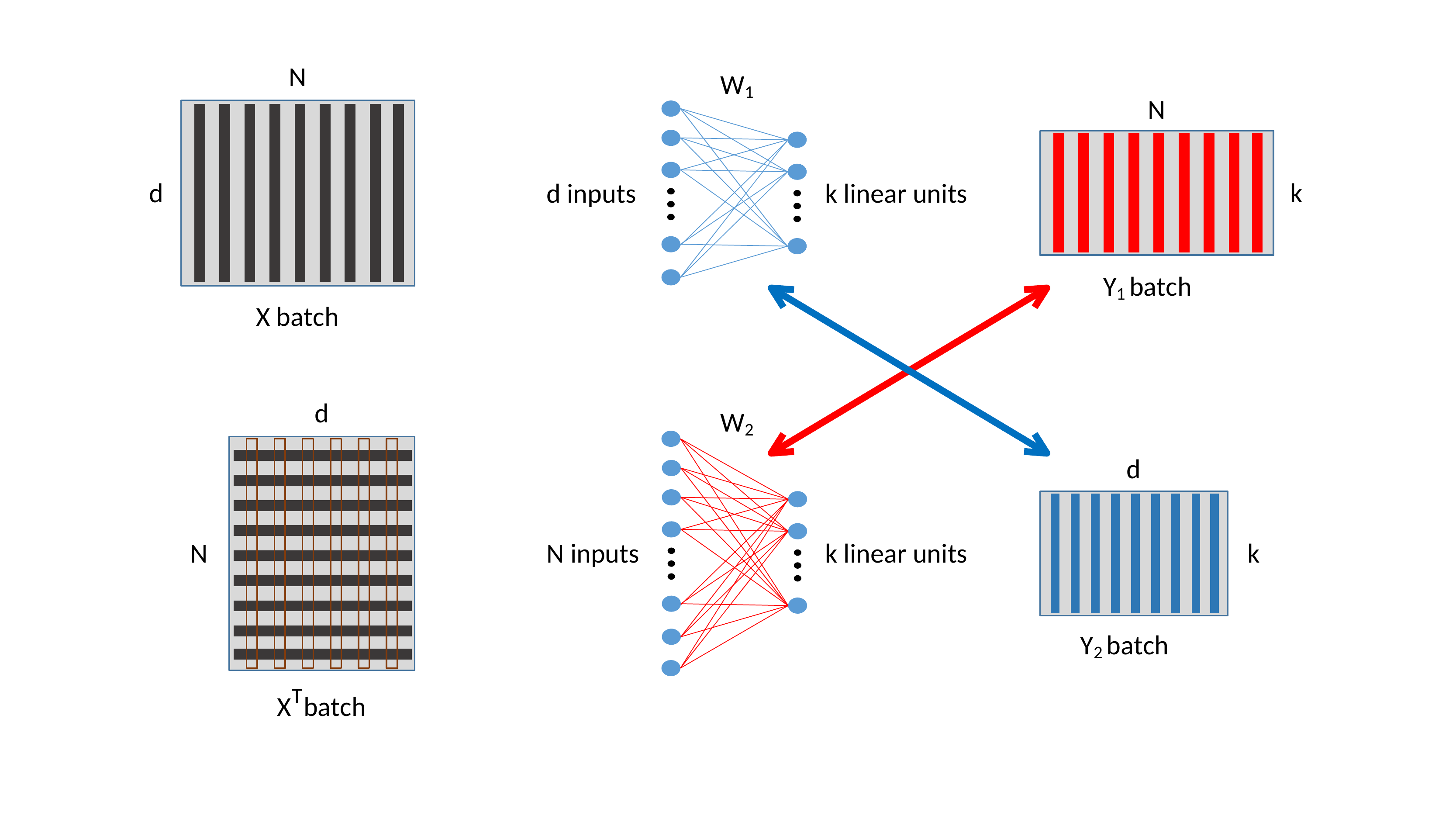}
    \caption{Dual linear single-layer neural networks.}
    \label{fig:theo-1}
\end{figure}

Consider a single layer neural network whose outputs have linear activation functions. There are $d$ input units and $k$ output units which represent a continuous signal in case of regression or class membership (posterior probabilities in case of cross entropy error function) in case of classification. A batch of $n$ samples, say $X$, is fed to the network. The weight matrix is $W_1$, where the  element $w_{ij}$ represents the weight from the input unit $j$ to the neuron $i$. 
The single layer neural network with linear activation functions in the lower scheme is here called the dual network of the former one. It has the same number of outputs and $n$ inputs. It is trained on the transpose of the original $X$ database. Its weight matrix is $W_2$ and the output batch is $Y_2$.
The following theorems state the duality conditions of the two architectures.
Figure \ref{fig:theo-1} represents the two networks and their duality.

\begin{theorem*}[Network duality in competitive learning] \label{thm:duality}
Given a loss function for competitive learning based on prototypes, a single linear network (base) whose weight output neurons are the prototypes is equivalent to another (dual) whose outputs are the prototypes, under the following assumptions:
\begin{enumerate}
    \item the input matrix of the dual network is the transpose of the input matrix of the base network;
    \item the samples of the input matrix $X$ are uncorrelated with unit variance
\end{enumerate}
\end{theorem*}

\begin{proof}
Consider a loss function based on prototypes, whose minimization is required for competitive learning.
From the assumption on the inputs (rows of the matrix $X$), it results $X X^T=I_d$. A single layer linear network is represented by the matrix formula:
\begin{equation}
    Y = W X = \Big[ \textrm{prototype}_1 \dots \textrm{prototype}_k \Big] X
\end{equation}
By multiplying on the right by $X^T$, it holds:
\begin{equation}
    W X X^T = Y X^T
\end{equation}
Under the second assumption:
\begin{equation}
    W = \Big[ \textrm{prototype}_1 \dots \textrm{prototype}_k \Big] = Y X^T
\end{equation}
This equation represents a (dual) linear network whose outputs are the prototypes $W$. Considering that the same loss function is used for both cases, the two networks are equivalent.
\end{proof}

This theorem directly applies to the GBC (base) and DGBC (dual) neural networks if the assumption $2$ holds for the training set. If not, a preprocessing, e.g. batch normalization, can be performed.

\begin{theorem*}[Impossible complete duality]
Two dual networks cannot share weights as  $W_1=Y_2$ and  $W_2=Y_1$ (complete dual constraint), except if the samples of the input matrix $X$ are uncorrelated with unit variance.
\end{theorem*}

\begin{proof}
From the duality of networks and their linearity, for an entire batch it follows:
\begin{eqnarray}
    \begin{cases}
    Y_1 &= W_1 X \\
    Y_2 &= W_2 X^T
    \end{cases}
    &\implies& W_1 = Y_1 X^T \nonumber \\ 
    &\implies& W_1 = W_1 X X^T \nonumber \\ 
    &\implies& X X^T = I_d
\end{eqnarray}
\begin{eqnarray}
    \begin{cases}
    Y_1 &= W_1 X \\
    Y_2 &= W_2 X^T
    \end{cases}
    &\implies& W_2 = Y_2 X^T \nonumber \\ 
    &\implies& W_2 = W_2 X^T X \nonumber \\ 
    &\implies& X^T X = I_n
\end{eqnarray}
where $I_d$ and $I_n$ are the identity matrices of size $d$ and $n$, respectively. These two final conditions are only possible if the samples of the input matrix $X$ are uncorrelated with unit variance, which is not the case in (almost all) machine learning applications.
\end{proof}
\begin{theorem*}[Half duality I]
Given two dual networks, if the samples of the input matrix $X$ are uncorrelated with unit variance and if  $W_1=Y_2$ (first dual constraint), then  $W_2=Y_1$ (second dual constraint).
\end{theorem*}

\begin{proof}
From the first dual constraint (see Figure \ref{fig:theo-3}, right), for the second network it stems:
\begin{equation}
    Y_2 = W_1 = W_2 X^T
\end{equation}
Hence:
\begin{equation}
    Y_1 = W_1 X \implies Y_1 = W_2 X^T X
\end{equation}
under the second assumption on $X^T$ from Theorem \ref{thm:duality}, which implies $X^T X = I_n$, the result follows (see Figure \ref{fig:theo-3}, left).
\end{proof}

\begin{theorem*}[Half duality II] \label{theo:HD2}
Given two dual networks, if the samples of the input matrix $X$ are uncorrelated with unit variance and if  $W_2=Y_1$ (second dual constraint), then  $W_1=Y_2$ (first dual constraint).
\end{theorem*}
\begin{proof}
From the second dual constraint (see Figure \ref{fig:theo-3}, left), for the second network it stems:
\begin{equation}
    Y_1 = W_2 = W_1 X
\end{equation}
From the assumption on the inputs (rows of the matrix $X$), it results  $X X^T = I_d$. The first neural architecture yields (see Figure \ref{fig:theo-3}, right):
\begin{equation}
    Y_2 = W_2 X^T \implies Y_2 = W_1 X X^T = W_1
\end{equation}
\end{proof}
Theorem \ref{theo:HD2} justifies the use of the first single-layer neural network as a competitive layer.

\begin{figure}[t]
    \centering
    \includegraphics[width=0.7\columnwidth,trim={8cm 1cm 6cm 3.5cm},clip]{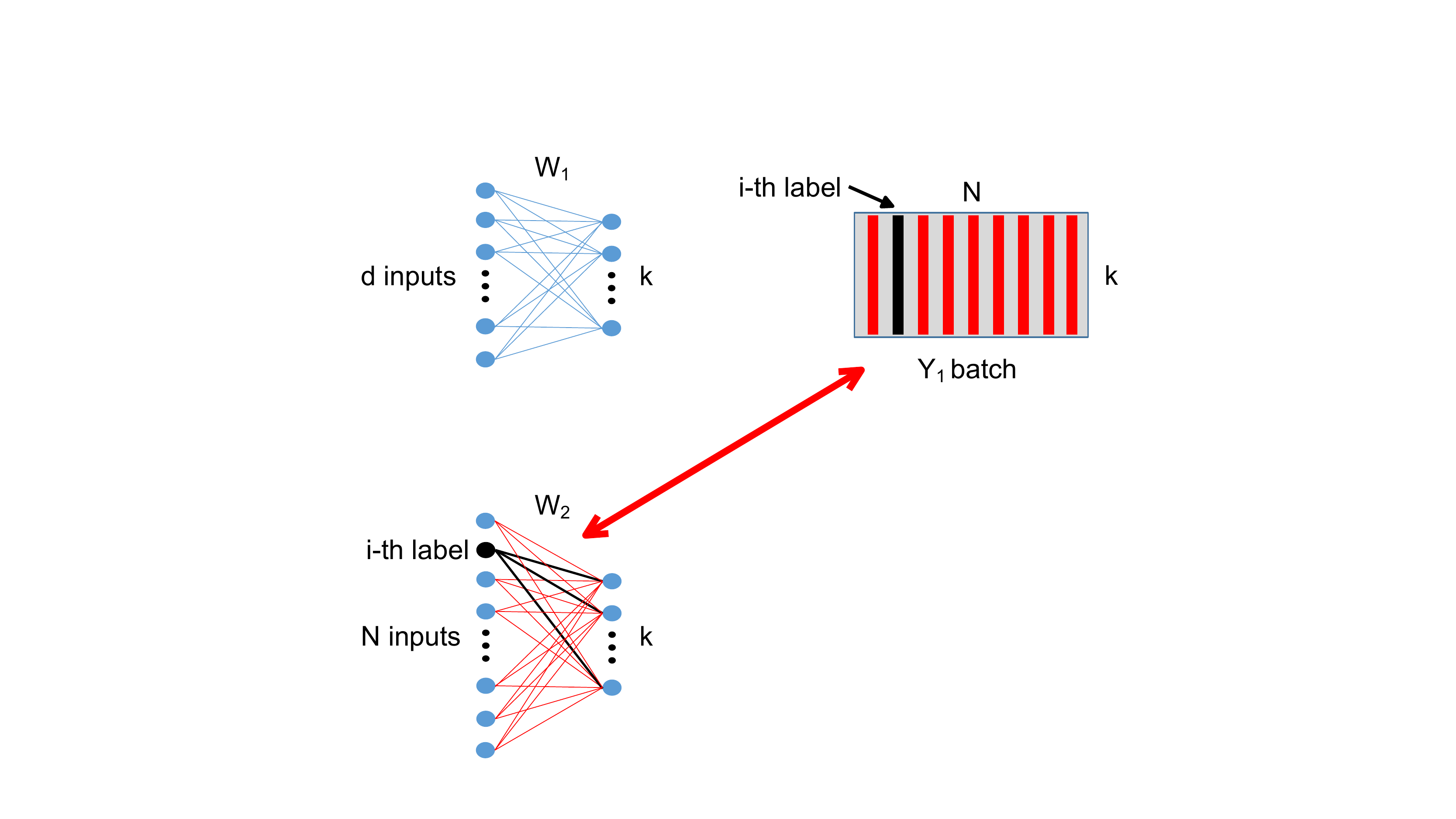}
    \includegraphics[width=0.7\columnwidth,trim={8cm 2cm 5.2cm 1cm},clip]{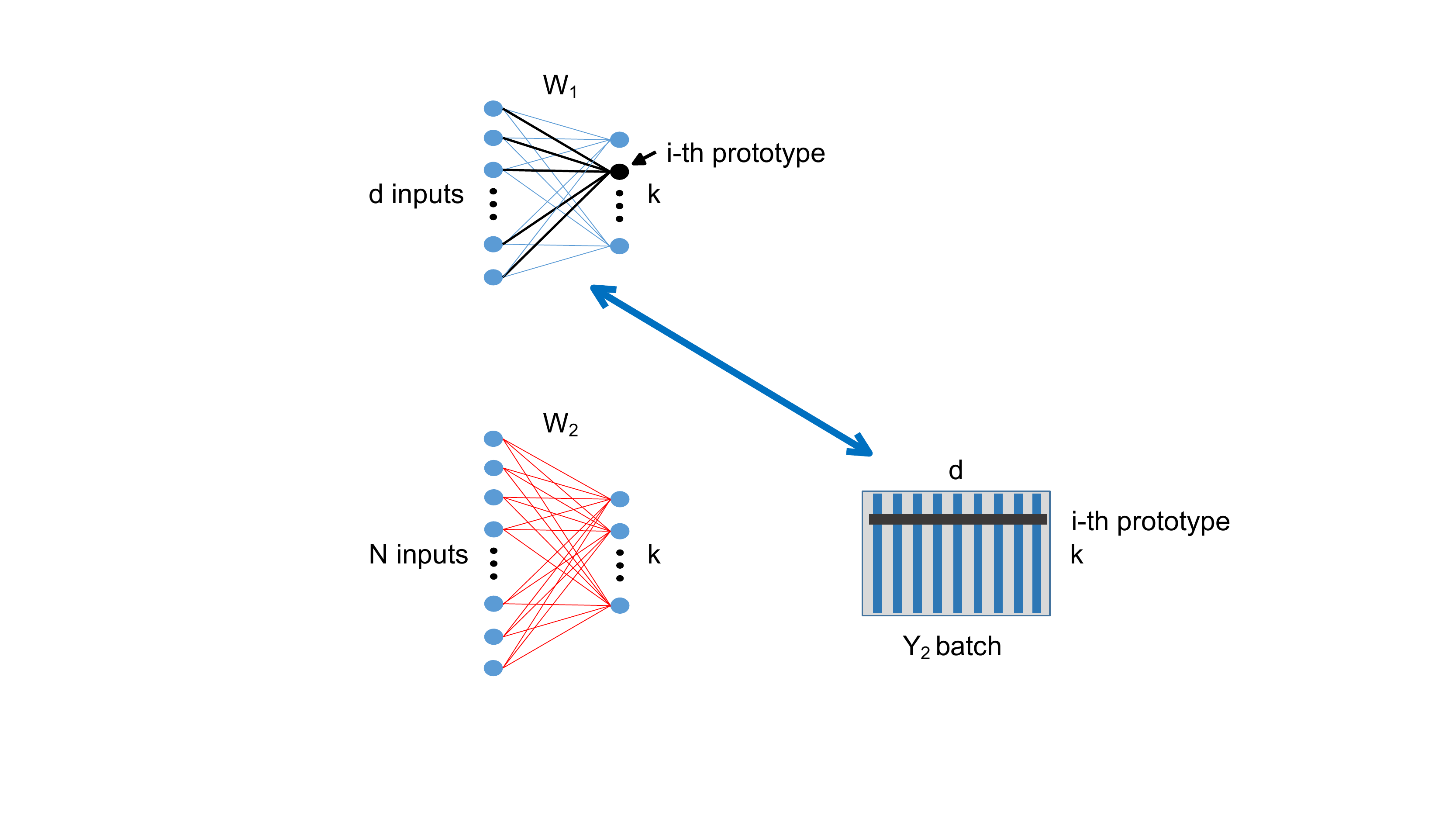}
    \vskip -0.5cm
    \caption{Half dualities.}
    \vskip -0.3cm
    \label{fig:theo-3}
\end{figure}

% \begin{corollary*}[Self-supervised learning] \label{corollary:self}
% The assumption of Theorem \ref{theo:HD2} implies the construction of labels for the base network.
% \end{corollary*}
%\begin{proof}
% As sketched in Figure \ref{fig:theo-3}, under the assumption of the equivalence between the training of the dual network (building of prototypes) and the architecture of the base network (output neurons as prototypes), the previous theorem implies the second dual constraint, which means the construction of a self-organized label.
% \end{proof}
% Thanks to this corollary, the base network can work in a self-supervised way, by using the results of the dual self-organization, to infer information on the dataset. This results in a new approach to self-supervised learning.
\subsection{Analysis of the learning process} \label{sec:analysis}
The theorems illustrated in the last section establish a set of conditions under which a base competitive layer (e.g. GBC) and its dual network (e.g. DGBC) are equivalent. However, this theory shows such an equivalence only in terms of the architecture of the two neural networks. By analyzing the forward and the backward pass, the learning process of the two layers is quite different. In particular, in the GBC layer there is no forward pass as $Y_1$ is not computed nor considered and the prototype matrix is just the weight matrix $W_1$:
\begin{equation}
    P_1 = \Big[\textrm{prototype}_1, \dots, \textrm{prototype}_k \Big] = W_1
\end{equation}
where $\textrm{prototype}_i \in \mathbb{R}^{d \times 1}$.
In the dual network, instead, the prototype matrix corresponds to the output $Y_2$; hence, the forward pass is a linear transformation of the input $X^T$ through the weight matrix $W_2$:
\begin{eqnarray}
P_2 &=& \Big[\textrm{prototype}_1 \dots \textrm{prototype}_k \Big]^T = Y_2 =  W_2 X^T = \nonumber \\
% &=& \begin{bmatrix}
% \mathbf{w}_1^T \\
% \mathbf{w}_2^T \\
% \dots \\
% \mathbf{w}_k^T
% \end{bmatrix}
% \begin{bmatrix}
% \mathbf{f}_1 & \mathbf{f}_2 & \dots & \mathbf{f}_d \\
% \end{bmatrix} = \nonumber \\
&=&
\begin{bmatrix}
\mathbf{w}_1^T \mathbf{f}_1 & \mathbf{w}_1^T \mathbf{f}_2 & \dots & \mathbf{w}_1^T \mathbf{f}_d \\
\mathbf{w}_2^T \mathbf{f}_1 & \mathbf{w}_2^T \mathbf{f}_2 & \dots & \mathbf{w}_2^T \mathbf{f}_d \\
\dots & \dots & \ddots & \vdots \\
\mathbf{w}_k^T \mathbf{f}_1 & \mathbf{w}_k^T \mathbf{f}_2 & \dots & \mathbf{w}_k^T \mathbf{f}_d \\
\end{bmatrix}
\end{eqnarray}
where $\mathbf{w}_i$ is the weight vector of the $i$-th output neuron of the dual network and $\mathbf{f}_i$ is the $i$-th feature over all samples of the input matrix $X$.
The components of each prototype are computed using a constant weight $\mathbf{w}_i$, because $P_2$ is an outer product, which has rank $1$. Besides, each component is computed as it were a one dimensional learning problem. For instance, the first component of the prototypes is $\Big[ \mathbf{w}_1^T \mathbf{f}_1 \dots \mathbf{w}_k^T \mathbf{f}_1 \Big]^T$; which means that the first component of all the prototypes is computed by considering just the first feature $\mathbf{f}_1$. Hence, each component is independent from all the other features of the input matrix, allowing the forward pass to be just like a collection of $d$ one-dimensional problems.

Such differences in the forward pass have an impact on the backward pass as well, even if the form of the loss function is the same for both systems. However, the parameters of the optimization are not the same. For the base network:
\begin{equation}
    \mathcal{L} = \mathcal{L} (X,W_1)
\end{equation}
while for the dual network:
\begin{equation}
    \mathcal{L} = \mathcal{L} (X^T,Y)
\end{equation}
where $Y$ is a linear transformation (filter) represented by $W_2$.
In the base competitive layer the gradient of the loss function with respect to the weights $W_1$ is computed directly as:
\begin{eqnarray}
\nabla \mathcal{L} (W_1) = \frac{d\mathcal{L}}{dW_1}
\end{eqnarray}
On the other hand, in the dual competitive layer, the chain rule is required to computed the gradient with respect to the weights $W_2$ as the loss function depends on the prototypes $Y_2$:
\begin{eqnarray}
\nabla \mathcal{L} (W_2) = \frac{d\mathcal{L}}{dW_2} = \frac{d\mathcal{L}}{dY_2} \cdot \frac{dY_2}{dW_2}
\end{eqnarray}
As a result, despite the architecture of the two layers is equivalent, the learning process is quite different.

\subsection{Qualitative analysis and comparison}

The differences outlined in the previous subsection %Section \ref{sec:analysis} 
may have an impact in favoring one or the other layer depending on the problem. The main advantage of using the base competitive layer consists in a lower computational cost, as the forward pass is not required and the backward pass is much easier to compute. Besides, for low dimensional datasets, i.e. when $N \gg d$, the size of the weight matrix $W_1$ is $d \times k$, while the size of $W_2$ is $N \times k$. This means that the number of parameters of the dual network is much higher with respect to the base layer, leading to a even higher computational cost. However, by having more parameters, the dual layer may have an advantage in terms of flexibility and in finding better local minima. 
On the other hand, in high-dimensional settings, when $N \ll d$, the matrix $W_1$ is much larger than $W_2$. Hence, by having fewer parameters to optimize, the dual layer behaves like a system with a larger set of constraints, leading to smoother gradient flows and less overfitting. 

Furthermore, another reason why the dual layer might be less sensitive to the number of features may depend on its learning process. Indeed, the forward pass decouples the original problem into a set of $d$ one-dimensional problems, while the loss function and the gradient perform the coupling of such problems. Finally, by considering how the two layers build their prototypes, the dual network seems more suitable for joining with deep neural networks. Indeed, the base network is an atypical layer as it does not perform a forward pass at all. The dual network, instead, is more similar to a regular layer as it applies a linear transformation to its input. This latter linear map could also be generalized to a nonlinear transformation by stacking a set of dense layers with nonlinear activation functions.

\subsection{Deep dual clustering}
The fact that the dual layer is designed for using the gradient of the loss function for training allows to back-propagate to other previous layers in order to preprocess implicitly the training set. The same can be said for the base network. However, the latter is not a true layer. Indeed, it is simply a minimization process in which the weights are directly estimated. For this reason, the output has no meaning. Instead, the dual one has meaningful outputs and, so, has the same nature of the blocks composing a deep neural network. 

The deep dual network is composed of a stack of fully-connected layers. The first layer is fed with the transpose of the input matrix X and the last layer is the dual linear layer. All the hidden layers have non linear activation functions (\textit{tanh}), but the output layer is linear. This approach allows the invariance of the feature dimension at each layer. Instead, it is the number of samples that changes at each step. In this way, clustering centroids are directly estimated in the original input space, despite the fact that they are pre-processed in the hidden layers.  

% {
% \color{red} $d$ trasformazioni lineari dove si utlizza la stessa trasformazione $T$ per tutti gli input
% \begin{itemize}
%     \item standard: prototype building by BP; NO transformation (layer=transformation) = NO layer = NO homogeneous for deep learning
%     \item dual: prototype building by transformation; transformation applied to 1D problem; BP couples 1D problems; suitable for deep learning
%     \item dual network insensitive to the number of features because the number of features only influences the transformation
%     \item BP+loss function sono un collante di tanti problemi 1D
%     \item i problemi 1D imply uncorrelatedness for each feature which implies a relaxation in data which implies more noise which implies a kind of simulated annealing which implies escaping from local minima
%     \item linear systems cannot diverge for all initial conditions
%     \item $N \gg d$ implies more parameters for the dual network which implies more flexibility
%     \item $N \ll d$ implies more constrained system for more complex problem which implies smoother solutions which implies less overfitting
% \end{itemize}
% }

\subsection{Extension to topological clustering}
\label{sec:experiments}

Topological clustering refers to a class of techniques where cluster centroids are connected during the learning process such that a Delaunay triangulation of the data manifold is induced. One of the most common approaches employs CHL at this purpose: If two prototypes are the two closest centroids for the same sample, an edge is created between them, representing their mutual relationship outlined by the common neighborhood.

% \subsection{Assembling the loss function}

\begin{figure*}[!ht]
    \centering

    \includegraphics[width=0.33\columnwidth]{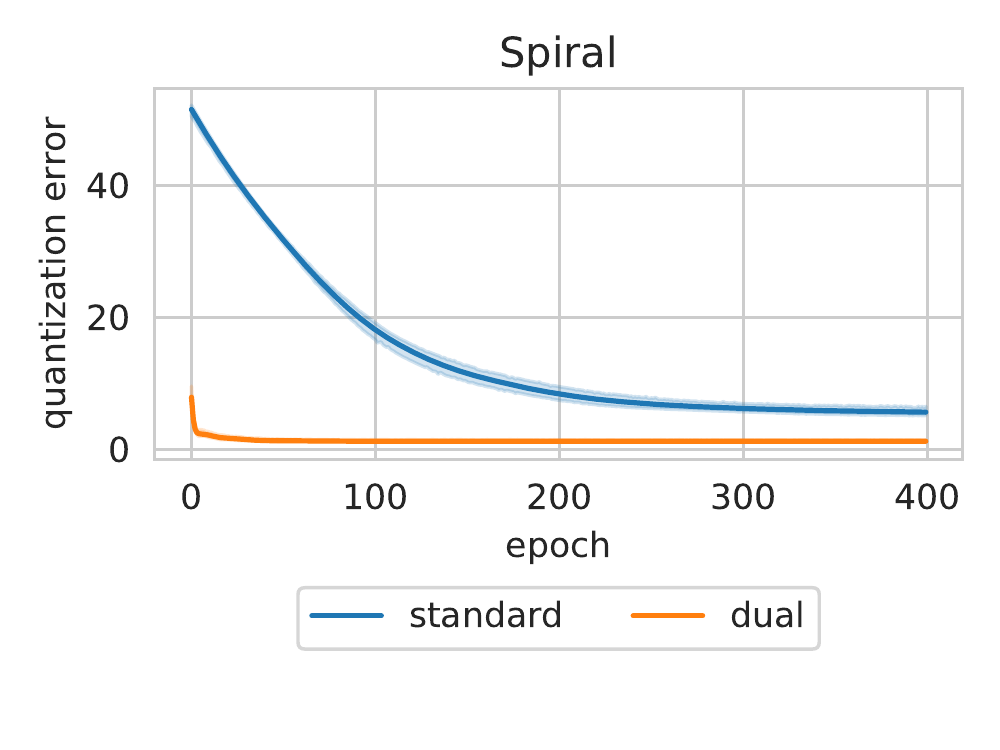}
    \includegraphics[width=0.33\columnwidth]{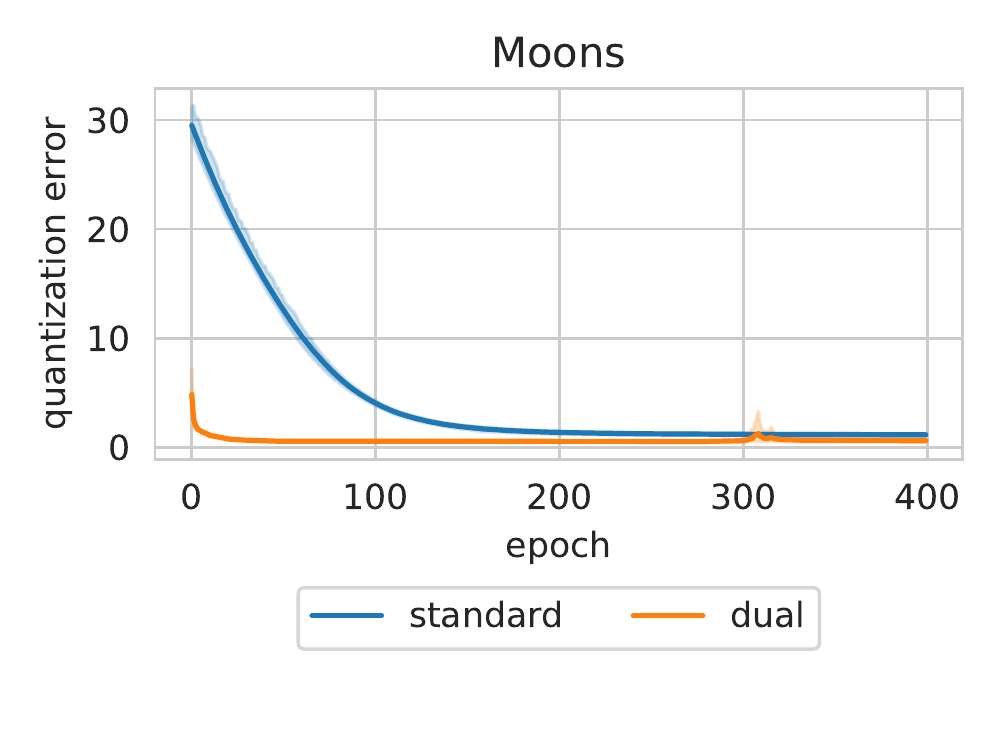}
    \includegraphics[width=0.33\columnwidth]{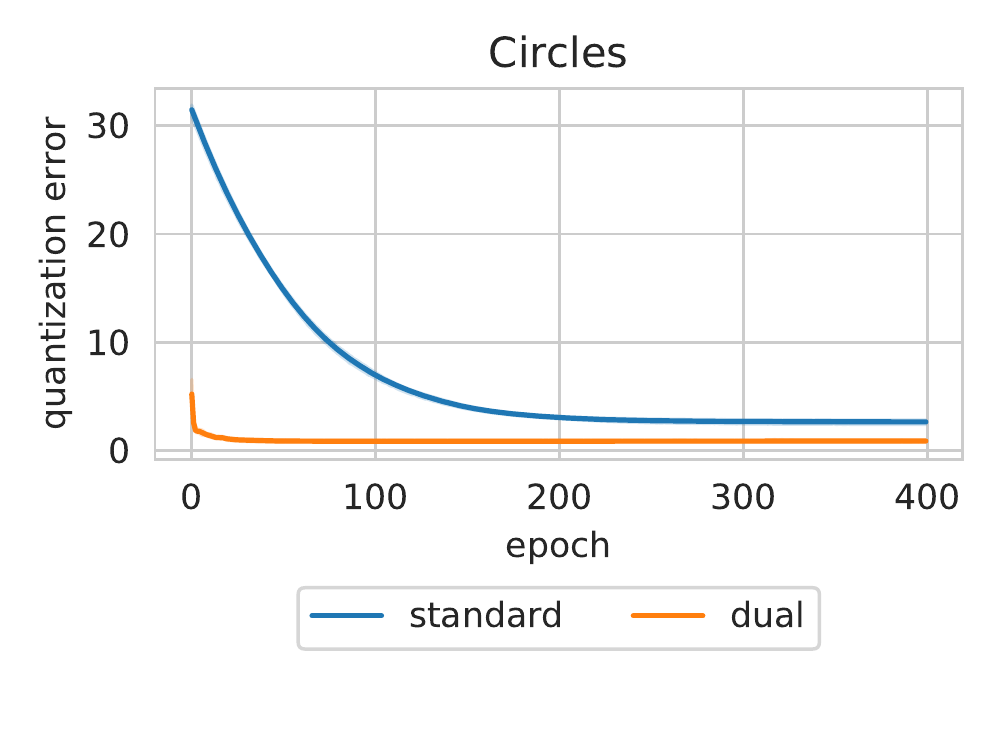}\\
    \vskip -0.85 cm
    \includegraphics[width=0.33\columnwidth]{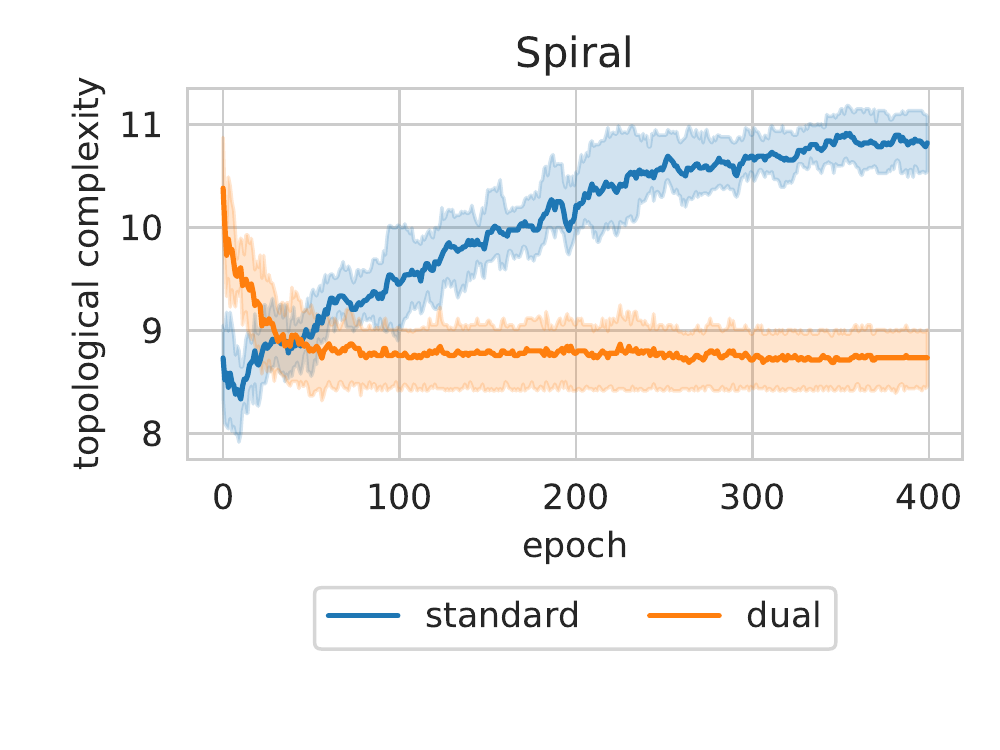}
    \includegraphics[width=0.33\columnwidth]{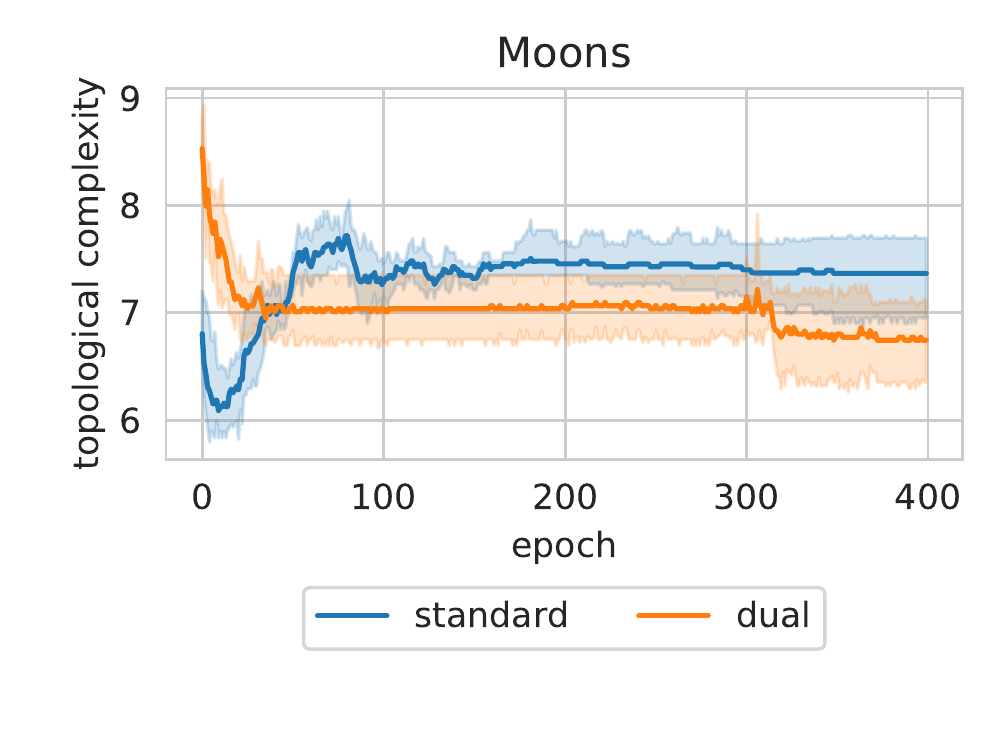}
    \includegraphics[width=0.33\columnwidth]{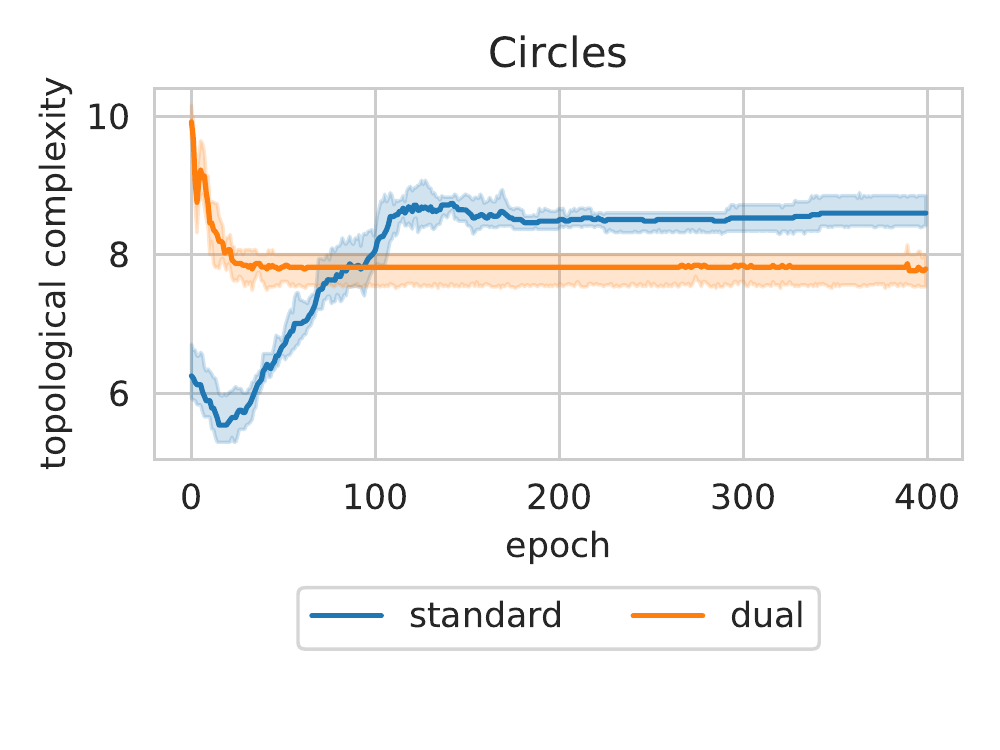}\\
    \vskip -0.85 cm
    \includegraphics[trim=0 50 0 0, clip, width=0.33\columnwidth]{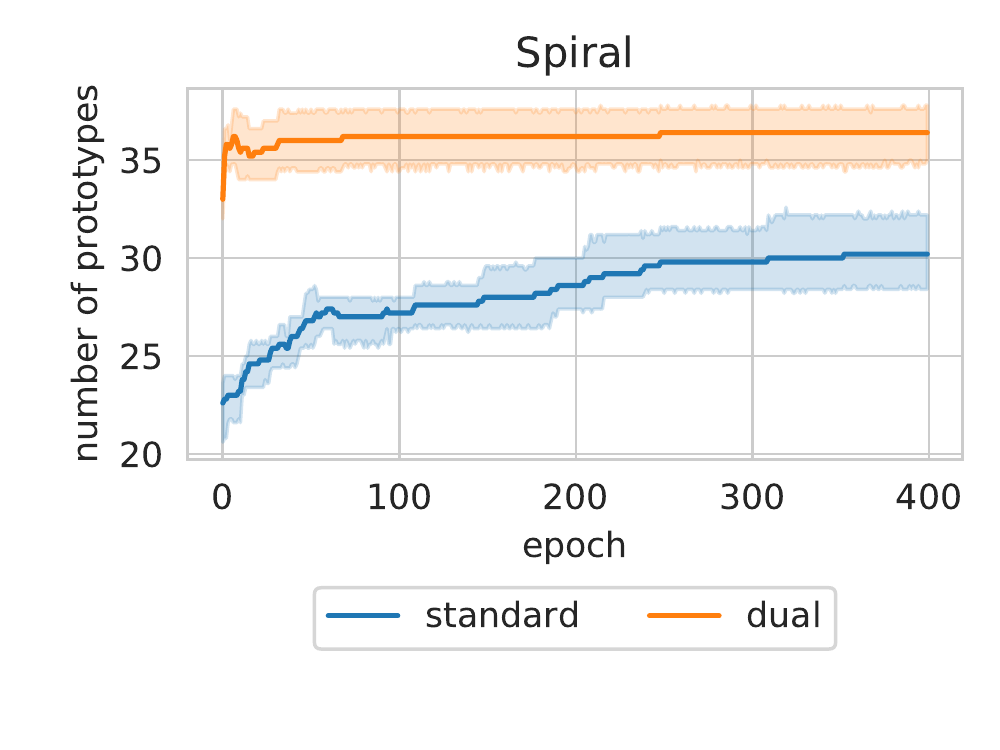}
    \includegraphics[trim=0 50 0 0, clip, width=0.33\columnwidth]{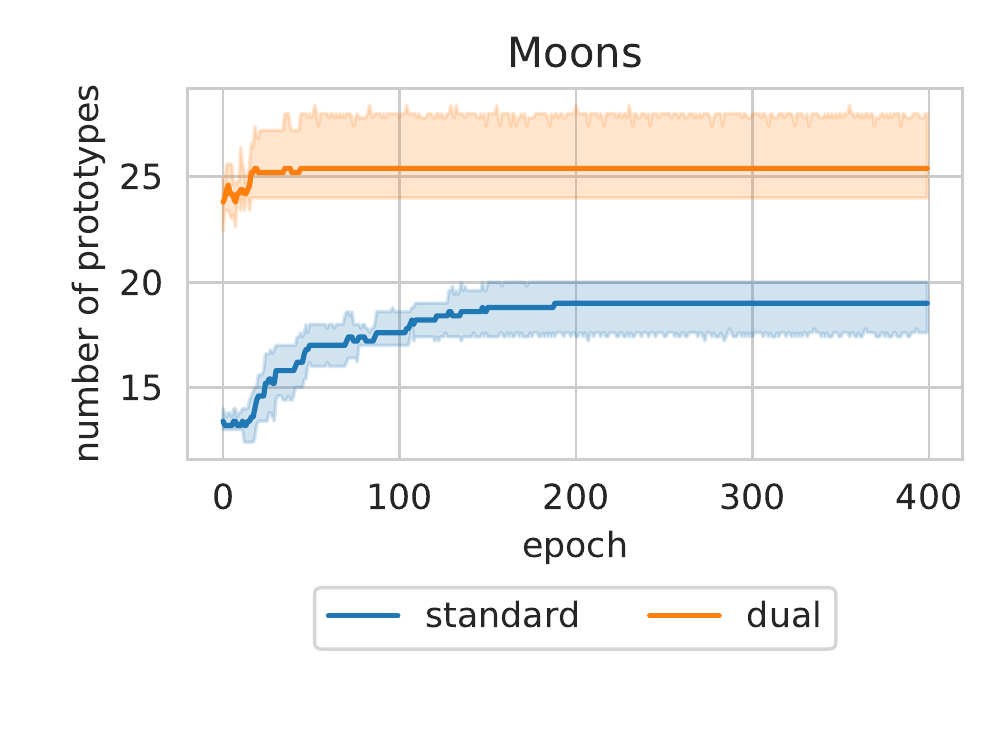}
    \includegraphics[trim=0 50 0 0, clip, width=0.33\columnwidth]{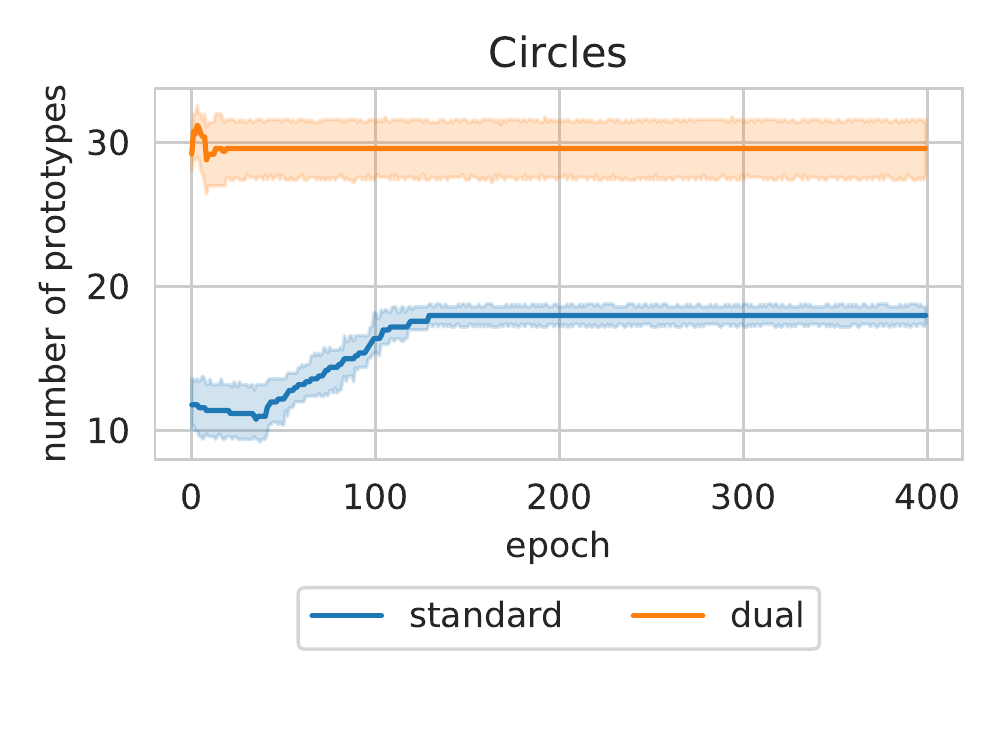}
    \vskip -0.1 cm
    \caption{Comparison GBC (\textbf{blue line}) and DGBC (\textbf{red line}) over $10$ runs on three metrics: the quantization error (\textbf{top row}), the norm of the matrix of the edges (\textbf{middle row}), and the number of valid prototypes (\textbf{bottom row}). The metrics are computed on three different datasets: \textit{Spiral} (\textbf{left column}), \textit{Moons} (\textbf{middle column}), and \textit{Circles} (\textbf{right column}).}
    \vskip -0.1 cm
     \label{fig:loss}
\end{figure*}

The framework developed in the previous section %Section \ref{sec:methods} 
can be easily adapted to accommodate for this kind of learning task. 
% Given a dataset $X \in \mathbb{R}^{n \times d}$, the dual network is defined as $f: \mathbb{R}^{d \times n} \rightarrow \mathbb{R}^{d \times k}$, where $n$ is the number of samples, $d$ the number of features, and $k$ represents an upper bound for the amount of prototypes. 
In this sense, the loss function main term represents a clustering index, the quantization error or the ratio between inter- and intra-cluster distances. However, in order to learn the minimal topological relationship, the loss function can be augmented by a Lagrangian term accounting for the complexity of the network connecting prototypes. At the end of each epoch, the adjacency matrix $E$, which represents the connections between prototypes using CHL, is computed and its norm is also included in the loss function. The gradient of the resulting loss can be computed in order to optimize prototypes' positions such that the complexity of the connections is minimized.
The overall loss function looks like:

\begin{equation}
    \mathcal{L} = \mathcal{Q} + \lambda ||E||_2
\end{equation}

where $\mathcal{Q}$ is the quantization error (average squared Euclidean distance between samples and corresponding centroids) and $E$ is the adjacency matrix representing the connections between prototypes. At the end of the learning process, prototypes without connections and with an empty Voronoi set can be pruned. Hence, the number of output units in the last layer represents an upper bound of the number of valid prototypes, as the neural network will automatically prune redundant centroids.

\section{Experimental evaluation}

In order to validate the theory with non-trivial experiments and to analyze the differences of the two learning approaches, the base competitive layer and its dual network are compared on three synthetic datasets containing clusters of different shapes and sizes.
Table \ref{tab:datasets} summarizes the main characteristics of each experiment.
The first dataset is composed of samples drawn from a two-dimensional Archimedean spiral (\textit{Spiral}). The second dataset consists of samples drawn from two half semicircles (\textit{Moons}). The last one is composed of two concentric circles (\textit{Circles}).
Each dataset is normalized by removing the mean and scaling to unit variance before fitting neural models. For all the experiments, the number of output units $k$ of the dual network is set to $30$. A grid-search optimization is conducted for tuning the hyperparameters. The
learning rate is set to $\epsilon=0.008$ for the base competitive layer and to $\epsilon=0.0008$ for its dual layer. Besides, for both networks, the number of epochs is equal to $\eta=400$ while the Lagrangian multiplier to $\lambda=0.01$.
For each dataset, both networks are trained $10$ times using different initialization seeds in order to statistically compare their performance.
All the code has been implemented in Python 3, relying upon open-source libraries \cite{abadi2016tensorflow,pedregosa2011scikit}. 
% and it is freely available under Apache 2.0 Public License from a GitHub repository\footnote{\url{https://github.com/pietrobarbiero/deep-topological-learning}}. The whole package can also be downloaded directly from PyPI\footnote{\url{https://pypi.org/project/deeptl/1.0.0/}}.
% All the experiments have been run on the same machine: Intel\textsuperscript{\textregistered} Core\texttrademark\ i7-8750H 6-Core Processor at 2.20 GHz equipped with 8 GiB RAM.

% \usepackage{booktabs}
% \usepackage{graphicx}
\begin{table}[!ht]
\renewcommand{\arraystretch}{1.5}
\centering
\caption{Synthetic datasets used for the experiments.}
\label{tab:datasets}
\begin{center}
\begin{sc}
% \resizebox{\columnwidth}{!}{%
\begin{tabular}{@{}lrrr@{}}
\toprule
dataset & samples & features & clusters \\ \midrule
Spiral & 500 & 2 & 1 \\
Moons & 500 & 2 & 2 \\
Circles & 500 & 2 & 2 \\
\bottomrule
\end{tabular}%
% }
\end{sc}
\end{center}
\vskip -0.1in
\end{table}

Qualitative results are presented in Figure \ref{fig:exp1}. The solutions provided by the base competitive layer are shown in the first row, while the dual network ones are in the second row. Nodes (prototypes) belonging to the same connected component are linked with edges according to their neighborhood. Samples are represented with different colors depending on the cluster they belong to. Qualitative considerations considering the location of prototypes suggest that good clustering performance can be obtained using both networks. However, as shown in Figure \ref{fig:loss}, the dual network tends to propose solutions using a slightly higher number of prototypes, thus finding better connections between them and providing a superior representation of the underlying topology, especially considering the \textit{Spiral} and the \textit{Circles} datasets, where clusters are well separated.

\begin{figure}[!b]
    \rotatebox{90}{$\ \ \ \ $\parbox{0.5cm}{\footnotesize{\textsc{GBC}}}}
    \includegraphics[width=0.315\columnwidth]{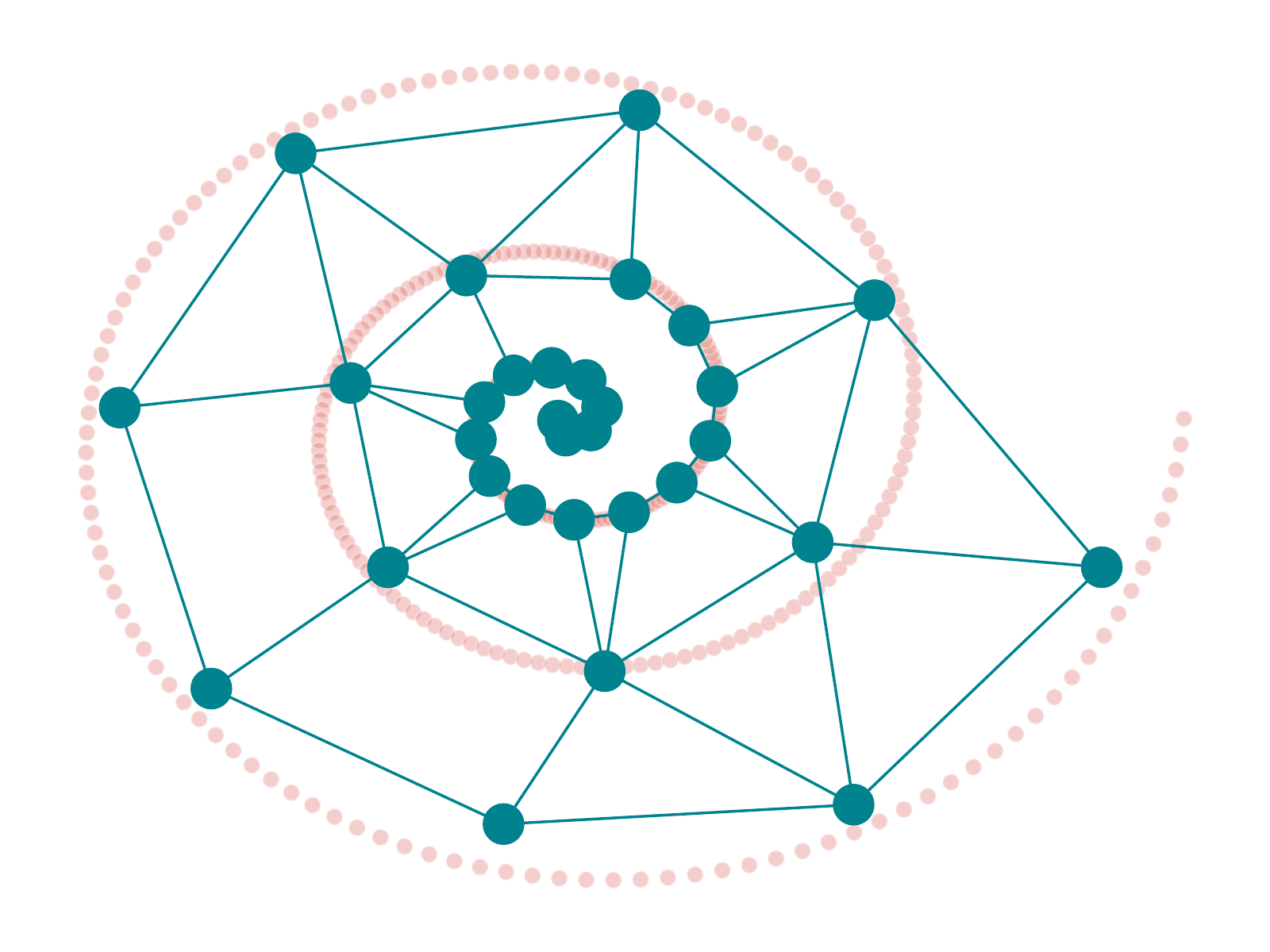}
    \includegraphics[width=0.315\columnwidth]{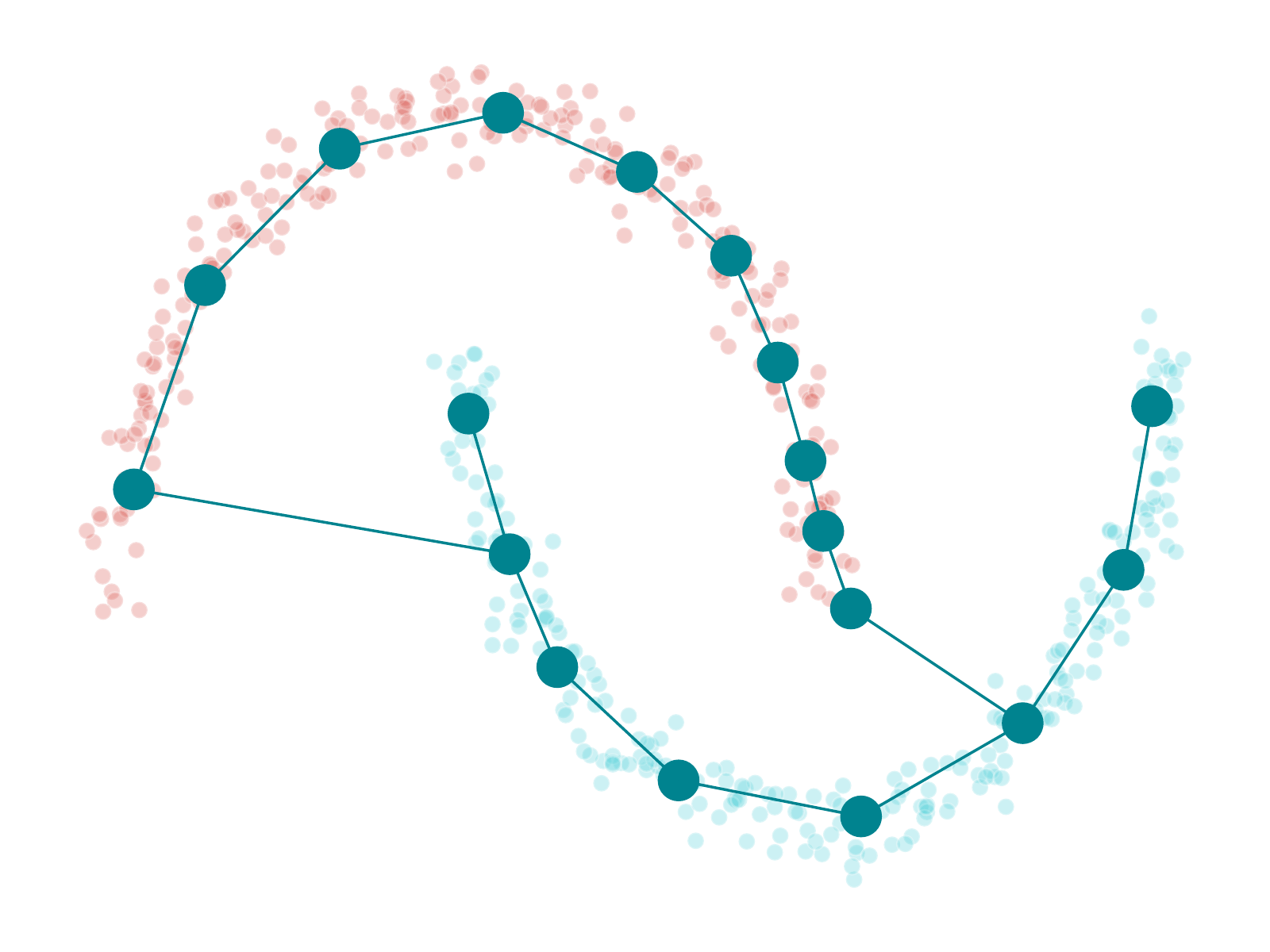}
    \includegraphics[width=0.315\columnwidth]{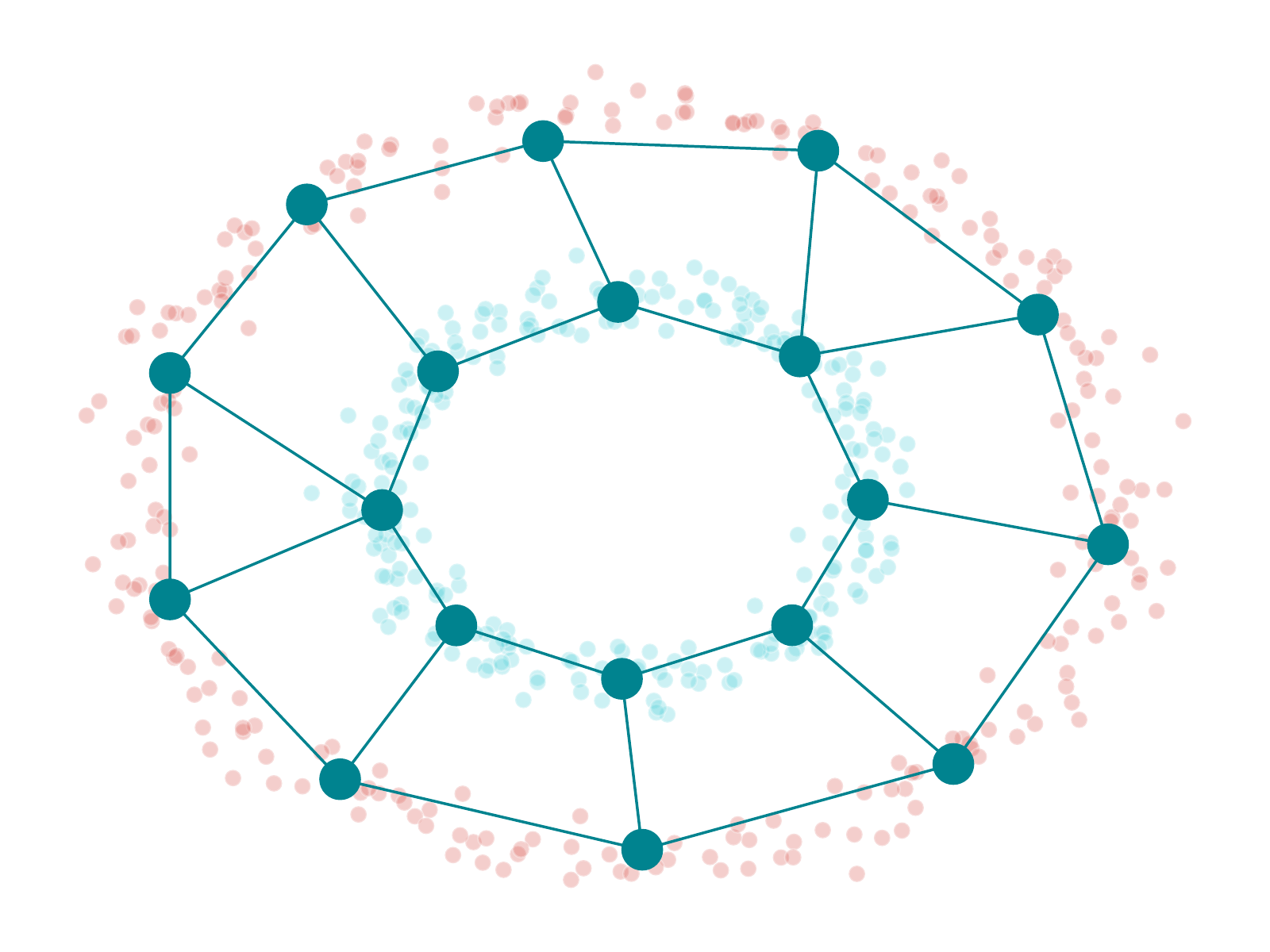}\\
    % \vskip -0.5cm
    \rotatebox{90}{$\ \ \ \ $\parbox{0.5cm}{\footnotesize{\textsc{DGBC}}}}
    \includegraphics[width=0.315\columnwidth]{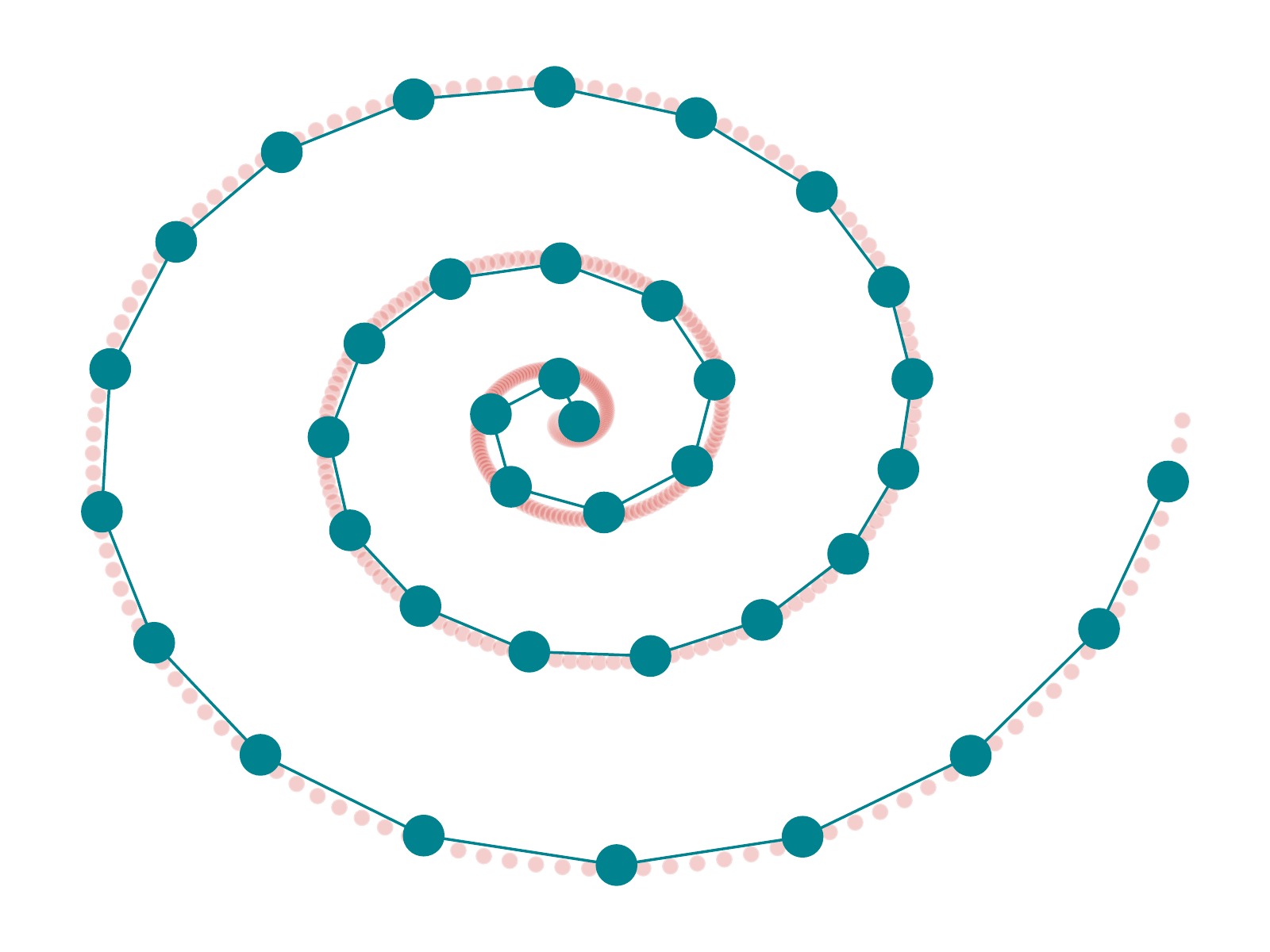}
    \includegraphics[width=0.315\columnwidth]{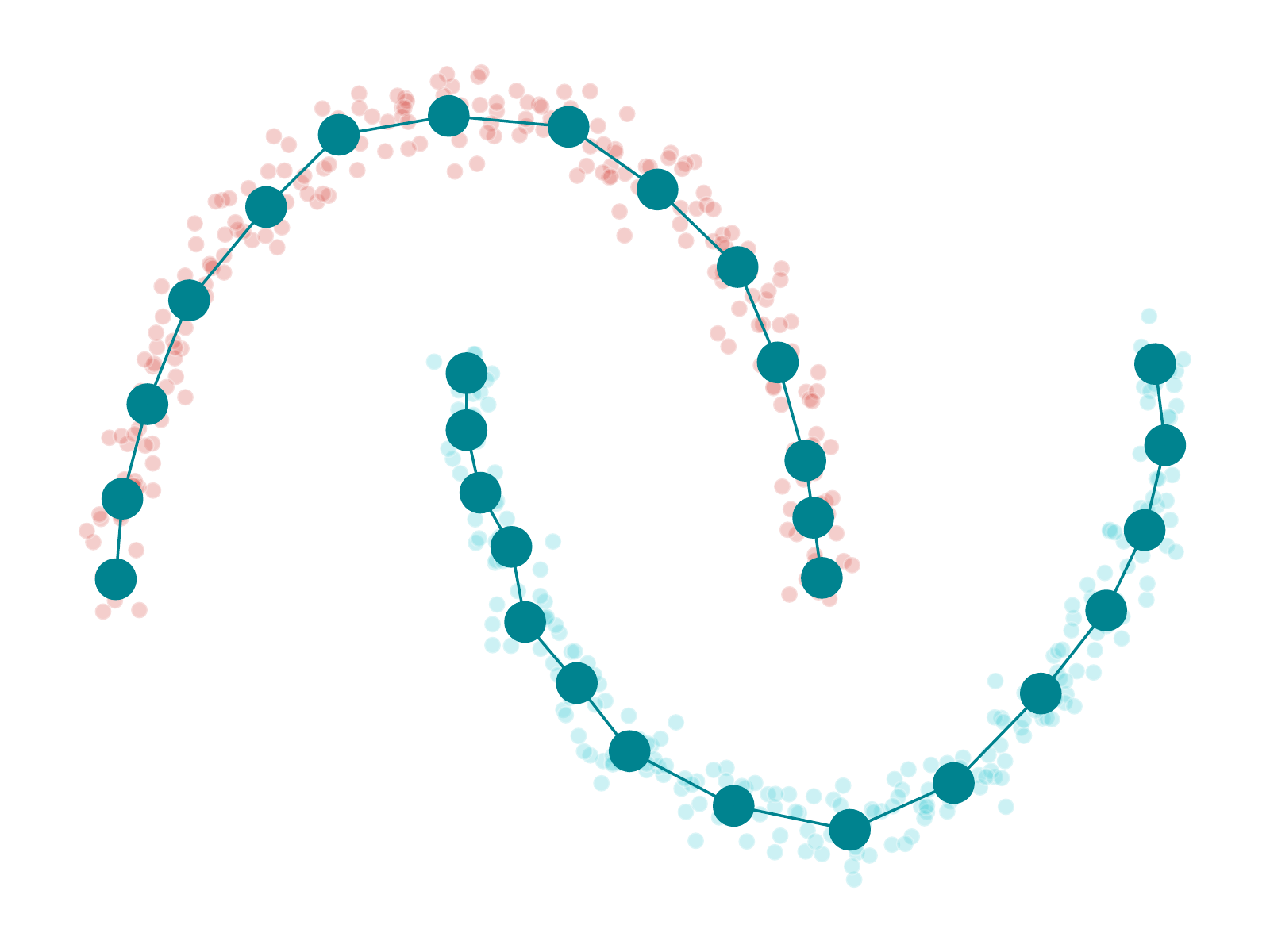}
    \includegraphics[width=0.315\columnwidth]{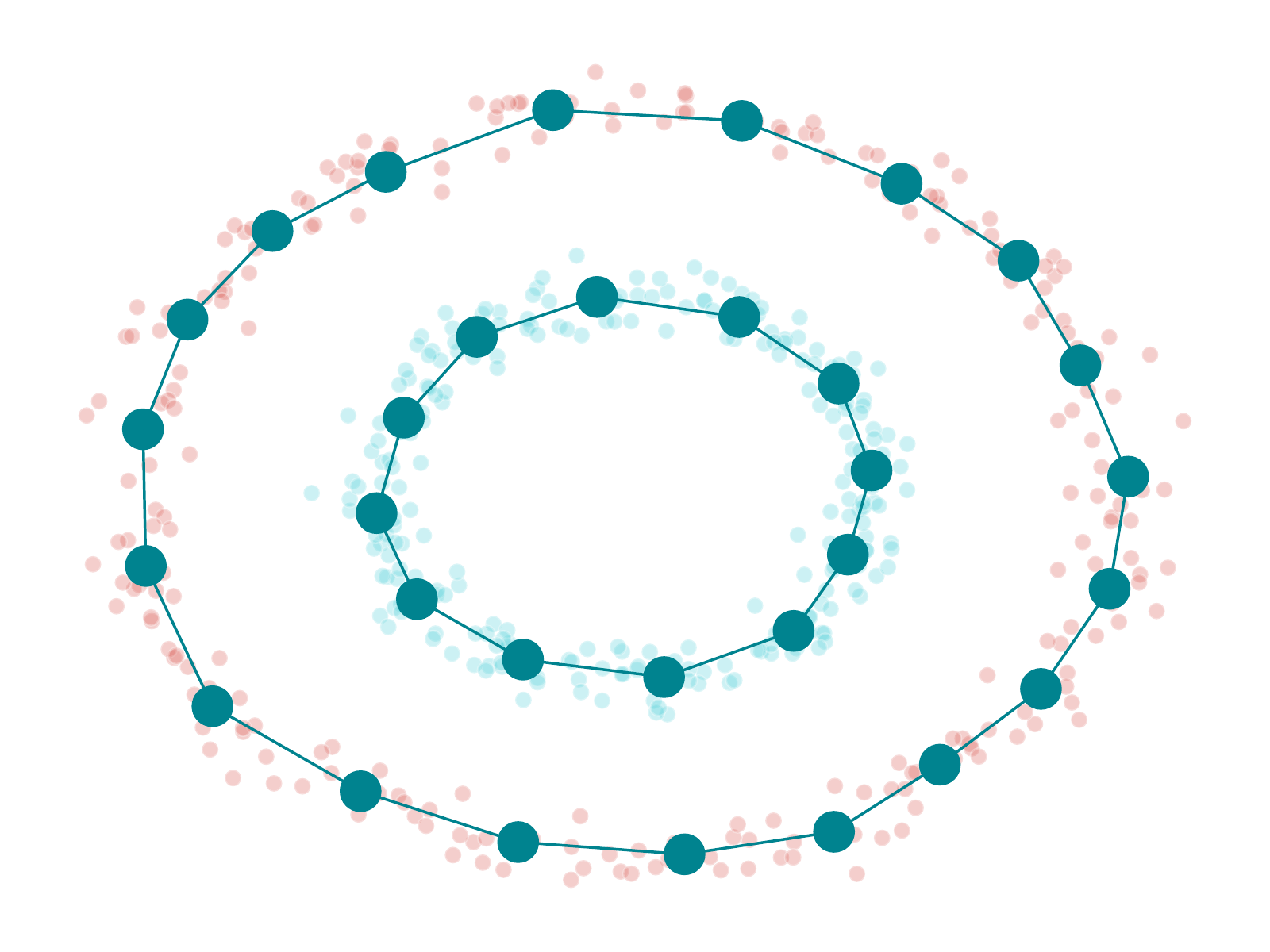}\\
    \vskip -0.5cm
    \caption{Experiments on synthetic datasets. From left to right: \textit{Spiral}, \textit{Moons}, and \textit{Circles} dataset.}
    \label{fig:exp1}
\end{figure}

In order to assess the main characteristics of the learning process, several metrics are evaluated while training the two networks on the three benchmark datasets. Figure \ref{fig:loss} shows for each dataset a comparison between the base layer and its dual on three key metrics: the quantization error, the topological complexity of the solution (i.e. the norm of the edge matrix), and the number of valid prototypes (i.e. the ones with a non-empty Voronoi set). 
The main differences between the two approaches are outlined by the quantization error.
Both networks seem to converge to similar local minima in all scenarios, thus validating their theoretical equivalence.
Nonetheless, the single-layer dual network exhibits a much faster rate of convergence compared to a standard competitive layer. The training of the dual network appears much more stable as outlined by a much lower variance of the quantization error. 
% By considering the topological complexity and the number of valid prototypes of the proposed solutions, the dual network seems to be much more flexible, providing solutions with higher complexity only when needed, especially for the \textit{Spiral} and the \textit{Circles} dataset (compare these results with Figure \ref{fig:exp1}).

\subsection{An application to high-dimensional clustering}
Here the performance of the standard competitive layer and its dual network in tackling high dimensional problems is assessed. Sure enough, standard distance-based algorithms generally suffer the well-known curse of dimensionality when dealing with high-dimensional data. Therefore, the intuition described in previous Section about dual-layer performances in this scenario is evaluated by working with an increasing number of features and a fixed number of samples. 
The MADELON algorithm proposed in \cite{guyon2003design} is used to generate the high-dimensional datasets. This algorithm creates clusters of points normally distributed about vertices of an n-dimensional hypercube. An equal number of cluster and data is assigned to two different classes. Both the number of samples ($n_s$) and the dimensionality of the space ($n_f$) in which they are placed can be defined programmatically. 
More precisely, the number of samples is set to $n_s=100$ while the number of features ranges in $n_f \in [1000, 2000, 3000, 5000, 10000]$. 
The number of required centroids is fixed to one tenth the number of input samples. Three different networks are compared: the base network (GBC layer), a single dual layer network (DGBC), and a deep variant of the dual network with two hidden layers of 10 neurons each (deep-DGBC). Results are averaged over $10$ repetitions on each dataset. Accuracy for each cluster is calculated by considering true positive those samples belonging to the class more represented and false positive the remaining data. As shown in the top plot of Fig. \ref{fig:high_dimensional}, GBC accuracy already drops when the number of feature is higher than $1000$. DBGC and deep-DBGC, instead, are more capable to deal with high-dimensional data and their accuracy remains near $100\%$ until $2000$ and $3000$, respectively. Nevertheless, all the proposed methods struggle when dealing with higher-dimensional data. 
% K-Means accuracy, at last, remains high even when dealing with $10^4$-dimensional data. 
% It is important to notice, however, that k-Means performances is considered as an upper bound for the proposed methods since the quantization found by this algorithm  is optimal; besides, it cannot be integrated within a neural network since it is not gradient-based.

\begin{figure}[th]
    \centering
    \includegraphics[trim= 10 0 10 25, clip, width=0.49\columnwidth]{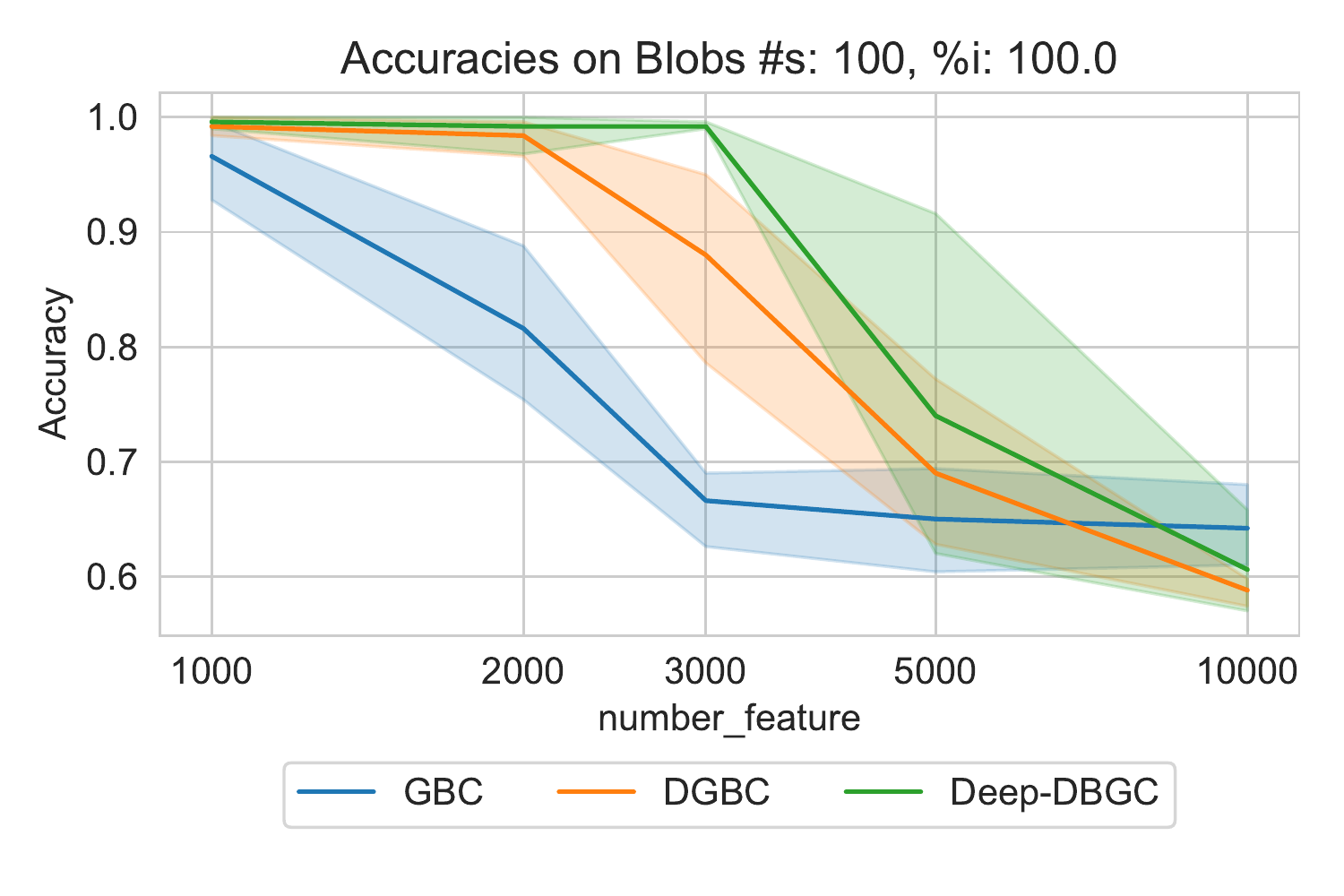}
    % \vskip -0.7 cm
    \includegraphics[trim= 10 50 10 25, clip, width=0.49\columnwidth]{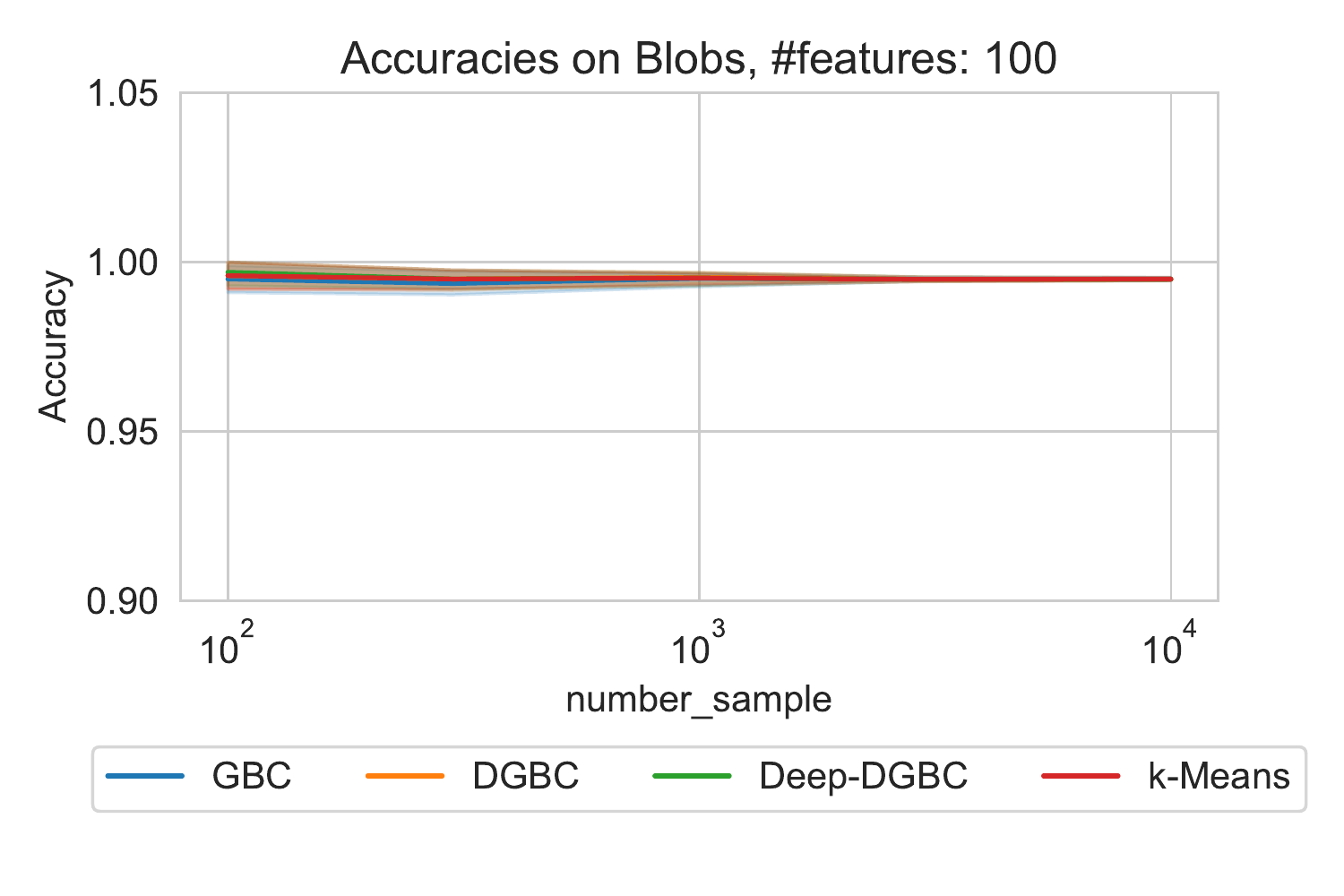}
    \vskip -0.3 cm
    \caption{Accuracy of the GBC and the DBGC layer, and the deep DBGC network tested while working with fixed number of samples and an increasing number of features (\textbf{top}) and a fixed number of features and an increasing number of samples (\textbf{bottom}) on the synthetic MADELON dataset \cite{guyon2003design}. Error bands represent the standard error of the mean.}
    \vskip -0.3 cm
    \label{fig:high_dimensional}
\end{figure}

A further experiment is also performed in order to check whether the opposite scenario holds true - i.e. that the DBGC layer was not suitable for working with a high number of samples (corresponding to a high number of network inputs). In order to do that we repeated the experiment on the MADELON dataset by setting a fixed number of features $n_f=100$, while working with an increasing number of samples $n_s \in [10^2, 10^3, 10^4]$. In the bottom plot of Fig.\ref{fig:high_dimensional}, it is shown that notwithstanding a higher computational complexity, DBGC and deep-DBGC are still capable to find a perfect quantization even when dealing with a very high number of samples.

\section{Conclusion}

This work sketches a novel interpretation of topological competitive learning using backpropagation. The foundation of a new theory is provided bridging two research fields which are usually thought as disjointed: gradient-based learning and unsupervised competitive neighborhood-based learning. This theory may represent the basis for a comprehensive reinterpretation of supervised and unsupervised learning with neural networks. 
Besides, as outlined in the experimental section, the framework can be easily extended to integrate complex topological structures and relationships among prototypes.
The two novel competitive layers presented in this work represent the first steps towards the integration of competitive and topological learning with deep neural architectures, outlining the power and flexibility of the approach paving the way towards more advanced and challenging learning tasks such as: topological nonstationary clustering \cite{randazzo2018nonstationary}, hierarchical clustering \cite{ghng,cirrincione2020gh}, core set discovery \cite{barbiero2020uncovering,ciravegna2019discovering}, incremental and attention-based approaches, or multi-objective optimization of a latent space with topological constraints.

\section*{Software}

To enable code reuse, the Python code for the mathematical models including parameter values and documentation is freely available under Apache 2.0 Public License from a GitHub repository
\footnote{\url{https://github.com/pietrobarbiero/deep-topological-learning}}. The whole package can also be downloaded directly from PyPI\footnote{\url{https://pypi.org/project/deeptl/1.0.0/}}.
Unless required by applicable law or agreed to in writing, software will be distributed on an "as is" basis, without
warranties or conditions of any kind, either express or implied.

\bibliographystyle{unsrt}  
\bibliography{references}

\end{document}